%% file: main_arxiv.tex
\documentclass{article}
\usepackage{fullpage,amsmath,amssymb,amsthm}

\newcommand{\cA}{\mathcal{A}}
\newcommand{\cX}{\mathcal{X}}
\newcommand{\cW}{\mathcal{W}}
\newcommand{\cD}{\mathcal{D}}
\newcommand{\cH}{\mathcal{H}}
\newcommand{\E}{\mathbb{E}}
\newcommand{\cQ}{\mathcal{Q}}
\newcommand{\cL}{\mathcal{L}}
\newcommand{\eps}{\varepsilon}
\DeclareMathOperator*{\sign}{sign}
\newcommand{\R}{\mathbb{R}}

\usepackage[algoruled,lined,noend,linesnumbered,algo2e]{algorithm2e}

\newtheorem{theorem}{Theorem}
\newtheorem{lemma}{Lemma}

\begin{document}

\title{The Impossibility of Parallelizing Boosting}

\author{Amin Karbasi\thanks{\texttt{amin.karbasi@yale.edu}. Yale University. Amin Karbasi acknowledges funding in direct support of this work from NSF (IIS-1845032), ONR (N00014- 19-1-2406), and the AI Institute for Learning-Enabled Optimization at Scale (TILOS).} \and Kasper Green Larsen\thanks{\texttt{larsen@cs.au.dk}. Aarhus University. Supported by Independent Research Fund Denmark (DFF) Sapere Aude Research Leader grant No 9064-00068B.}}

%\icmlaffiliation{aarhus}{Computer Science, Aarhus University, Aarhus, Denmark}
%\icmlaffiliation{yale}{Yale University, New Haven, USA}

%\icmlcorrespondingauthor{Amin Karbasi}{amin.karbasi@yale.edu}
%\icmlcorrespondingauthor{Kasper Green Larsen}{larsen@cs.au.dk}

\date{}
\maketitle

\begin{abstract}
The aim of boosting is to convert a sequence of weak learners into a strong learner. At their heart, these methods are fully sequential. In this paper, we investigate the possibility of parallelizing boosting. Our main contribution is a strong negative result, implying that significant parallelization of boosting requires an exponential blow-up in the total computing resources needed for training.
\end{abstract}

%\thispagestyle{empty}
%\newpage
%\setcounter{page}{1}

\input{intro_arxiv}

\input{lower_arxiv}

\input{upper_arxiv}

\input{concl_arxiv}

\input{ack_arxiv.tex}

\bibliography{res_arxiv}
\bibliographystyle{abbrv}

\newpage
\appendix
\input{appendix_arxiv}

\end{document}

%% file: intro_arxiv.tex
\section{Introduction}
Boosting is one of the most successful ideas in machine learning, allowing one to "boost" the performance of a base learning algorithm with rather poor accuracy into a highly accurate classifier, with recent applications in adversarial training \cite{abernethy2021multiclass}, reinforcement learning \cite{brukhim2021boosting}, and federated learning \cite{shen2022federated}, among many others.  The classic boosting algorithm, known as AdaBoost~\cite{adaboost}, achieves this by iteratively training classifers on the training data set. After each iteration, the data set is reweighed and a new classifier is trained using a weighted loss function. The weights intuitively guide the attention of the base learning algorithm towards training samples that the previous classifiers struggle with. After a sufficiently large number of iterations, the produced classifiers are combined by taking a weighted majority vote among their predictions.

Both the classic AdaBoost algorithm, as well as more modern gradient boosters~\cite{gradboost,lightGBM,xgboost}, all have this highly sequential behaviour, where the algorithm runs in multiple iterations that adjust the learning problem based on previously trained classifiers/regressors. Indeed, the best performance of gradient boosters on benchmark data sets is often obtained after hundreds, or even thousands of iterations~\cite{survey}. This may be appropriate when the base learning algorithm has a small training time. However, it prevents the use of boosting in combination with e.g. medium-sized neural networks as the base learning algorithm, or if one wishes to use all available training data in a large data set. Here the sequential nature of boosting algorithms is particularly critical, as it is not possible to simply distribute the training task to many machines. This shortcoming of boosting algorithms was also highlighted in the survey by Natekin and Knoll~\cite{survey} when discussing drawbacks of gradient boosting.

In light of the above concerns, it would have a huge practical impact if a highly parallel boosting algorithm could have been developed. Unfortunately, our main result shows that parallelizing boosting cannot be done without an \emph{exponential} increase in the total work needed for training!

\paragraph{Weak to Strong Learning.}
To formalize the above claim that boosting cannot be parallelized, we need to introduce the theoretical framework in which we prove our impossibility result.

Boosting was introduced to address a theoretical question by Kearns and Valiant~\cite{kearns1988learning,kearns1994cryptographic}, asking whether a so-called \emph{weak learner} can always be converted to a \emph{strong learner}. In the following, we define these notions formally. First, let $\cX$ be an input domain and $c : \cX \to \{-1,1\}$ an unknown concept that we want to learn. 

A $\gamma$-weak learner for a concept $c : \cX \to \{-1,1\}$, is an algorithm that given some constant number of samples $m_0$ from \emph{any} unknown distribution $\cD$ over $\cX$, with constant probability returns a hypothesis $h : \cX \to \{-1,1\}$ such that $\cL_\cD(h):=\Pr_{x \sim \cD}[h(x)\neq c(x)] \leq 1/2- \gamma$. It thus has an accuracy that is $\gamma$ better than guessing. We say that the weak learner has a $\gamma$ \emph{advantage}. We remark that whether or not the number of samples $m_0$ is allowed to depend on $\gamma$ is irrelevant for our results and they apply in all circumstances.

A strong learner on the other hand, is a learning algorithm such that for any $0 < \eps,\delta < 1$, there is some number of samples $m(\eps,\delta)$, such that when given $m(\eps,\delta)$ samples from any unknown distribution $\cD$ over $\cX$, with probability at least $1-\delta$, it returns a hypothesis $h : \cX \to \{-1,1\}$ such that $\cL_\cD(h) \leq \eps$. A strong learner can thus obtain arbitrarily high accuracy when given enough training data. We refer to $m(\eps,\delta)$ as the sample complexity of the strong learner.

Kearns and Valiant thus asked whether the rather poor accuracy of a weak learner can always be exploited to obtain a classifier with arbitrarily high accuracy, i.e. a strong learner. This question was answered affirmatively~\cite{schapire1990strength}, and AdaBoost is one example of such a weak-to-strong learning algorithm. Concretely, the tightest known bounds on the sample complexity of AdaBoost states that the hypothesis $h$ it produces on a training set $S \sim \cD^m$ of $m$ samples, when using a $\gamma$-weak learner $\cW$ as the base learning algorithm, satisfies
\begin{eqnarray}
\label{eq:ada}
\cL_\cD(h) = O\left(\frac{d \ln(m) \ln(m/d) + \ln(1/\delta)}{\gamma^2 m} \right)
\end{eqnarray}
with probability at least $1-\delta$ over $S$, see e.g.~\cite{understandingMachineLearning}.
Here $d$ denotes the VC-dimension of the hypothesis set $\cH$ that the weak learner $\cW$ outputs from. It is clear from this formula that the accuracy can be made arbitrarily high when given enough training data. In particular, for any $\eps, \delta$, if we ignore log-factors, then the sample complexity $m(\eps,\delta)$ grows as $O((d+ \ln(1/\delta))/\gamma^2 \eps)$. Furthermore, this sample complexity has recently been proved optimal for any weak-to-strong learning algorithm~\cite{optimalWeakToStrong}.

In weak-to-strong learning algorithms $\cA$, such as AdaBoost, a $\gamma$-weak learner $\cW$ is typically used by feeding it a distribution $\cD$ over the training data $S$ (e.g. the weighing over the data set in each iteration of AdaBoost). It is then assumed that $\cW$ returns a hypothesis $h$ with error at most $1/2-\gamma$ under $\cD$. Note that this seems quite different from the above definition of a $\gamma$-weak learner, where the weak learner receives at least $m_0$ samples from an unknown distribution $\cD$. However, the definition of a $\gamma$-weak learner actually implies that $\cW$ can compute a hypothesis $h$ with error at most $1/2-\gamma$ for a query distribution $\cD$ specified by $\cA$. This is because $\cW$ is given access to the training data $S$ and the distribution $\cD$ and thus it can just repeatedly sample $m_0$ samples from $\cD$, compute a hypothesis from the samples, and compute its exact error probability under $\cD$. If the error probability exceeds $1/2-\gamma$, $\cW$ can just repeat with a fresh sample. Since $\cW$ is a weak learner, this terminates in an expected constant number of tries. Here one critically exploits that a weak learner has advantage $\gamma$ under \emph{any} distribution when given at least $m_0$ samples. Henceforth, we thus think of a weak learner $\cW$ as a procedure that we can query with a training set $S$, the labels $c(S)$ and a distribution $\cD$ over $S$. It then returns a hypothesis $h$ with $\cL_\cD \leq 1/2-\gamma$. A weak-to-strong learning algorithm $\cA$ is then given query access to a $\gamma$-weak learner as well as the training set $S$ and the labels $c(S)$. After querying $\cW$ sufficiently many times to obtain hypotheses $h_1,\dots,h_k$, it must output a hypothesis $h : \cX \to \{-1,1\}$. Note that we make no assumption that the output $h$ is a voting classifier.

\paragraph{Parallel Weak to Strong Learning.}
We are now ready to define what we formally mean by parallelizing boosting. We say that a weak-to-strong learning algorithm $\cA$ has \emph{parallel complexity} $(p,t)$ if it only invokes its weak learner $\cW$ in $p$ rounds. In each round, it may query the weak learner with up to $t$ different distributions $\cD_1,\dots,\cD_t$. The weak learner then returns a hypothesis $h_i$ for each $\cD_i$ such that $\cL_{\cD_i}(h_i) \leq 1/2-\gamma$. The queries made in any round may only depend on $S, c(S)$ as well as the hypotheses obtained from $\cW$ in previous rounds. We let $\cA_{S,c(S),\cW}$ denote the hypothesis $\cA$ reports on input $S,c(S),\cW$ when the $p$ rounds are over.

In this terminology, it is known that AdaBoost needs $\Theta(\gamma^{-2} \ln m)$ iterations to obtain the accuracy claimed in~\eqref{eq:ada}. It thus has parallel complexity $(\Theta(\gamma^{-2} \ln m), 1)$. With parallel complexity defined, we are finally ready to present our main result, ruling out parallel boosting
\begin{theorem}
\label{thm:main}
There is a universal constant $a>0$ such that for any weak-to-strong learner $\cA$, any $m$, any $0 < \gamma < a$ and any VC-dimension $d$, there exists a domain $\cX$, a concept $c : \cX \to \{-1,1\}$ and a $\gamma$-weak learner $\cW$ for $c$, such that $\cW$ uses a hypothesis set of VC-dimension $d$ and either: $p \geq \exp(\Omega(d))$ or $t \geq \min\{\exp(\Omega(d/\gamma^2)), \exp(\exp(\Omega(d)))\}$ or
%\begin{eqnarray*}
$\cL_\cD(\cA_{S,c(S),\cW}) \geq
\exp(-O(p \max\{\gamma, \ln(tp)\gamma^2/d\}))
$
%\end{eqnarray*}
in expectation over $S$ and any random choices of $\cA$. Here $\cX, c$ and $\cW$ depends on $m, \gamma$ and $d$.
\end{theorem}
Let us discuss the implications of the theorem in detail. First, let us consider the simplest case where one wants only a constant number of parallel rounds $p = O(1)$. Then the lower bound states that either $t \geq \exp(\Omega(d/\gamma^2)), t \geq \exp(\exp(\Omega(d)))$ or the error probability is at least $\exp(-O(\max\{\gamma, \ln(t)\gamma^2/d\}))$. To make this error probability comparable to AdaBoost~\eqref{eq:ada} requires $t = \exp(\Omega(d \gamma^{-2} \ln m))$. Thus there is no way around either an exponential dependency on $d \gamma^{-2}$ or a double-exponential dependency on $d$.

If one is willing to use a super-constant number of rounds $p$, then there are a couple of possibilities for obtaining an error probability comparable to~\eqref{eq:ada}. First, one could have either $p =\exp(\Omega(d)), t = \exp(\Omega(d \gamma^{-2}))$ or $t = \exp(\exp(\Omega(d)))$. These bounds are all exponential in either the VC-dimension or $\gamma^{-2}$ or both. Finally, there is the possibility of making the error $e^{-O(p \gamma)}$ small. To make it comparable to~\eqref{eq:ada} requires $p = \Omega(\gamma^{-1} \ln m)$. This is only a $\gamma$ factor less than AdaBoost. Thus there is unfortunately not much hope for parallelizing boosting, and certainly not to a near-constant number of rounds. In all circumstances, if the number of rounds is significantly less than $\gamma^{-1} \ln m$, then it requires an exponential number of queries $t$ per round.

Finally, let us remark that it is sometimes stated that the unknown concept $c$ belongs to some concept class $\mathbb{C}$. Here we implicitly assume that it belongs to the class of all concepts that may be $\gamma$-weak learned using the hypothesis set of the weak learner. Alon et al.~\cite{DBLP:conf/stoc/AlonGHM21} proved that this concept class has VC-dimension at most $O_d(\gamma^{-2+2/(d+1)})$ when the hypothesis set of the $\gamma$-weak learner has VC-dimension $d$. Here $O_d(\cdot)$ hides factors depending only on $d$.

\paragraph{A Parallel Boosting Algorithm.}
To demonstrate the tightness of our lower bound, we also present a single-round boosting algorithm
\begin{theorem}
There is a weak-to-strong learner $\cA$, such that for any concept $c$, any $\gamma$-weak learner $\cW$ for $c$ using a hypothesis set of VC-dimension $d$, and any distribution $\cD$, when given $m$ samples $S \sim \cD^m$, $\cA$ has parallel complexity $(1,\exp(O(d\gamma^{-2} \ln m)))$ and with probability $1-\delta$ over $S$, it outputs a hypothesis $h$ with
\[
\cL_\cD(h) = O\left(\frac{d \ln(m) \ln(m/d) + \ln(1/\delta)}{\gamma^2 m} \right).
\]
\end{theorem}
The generalization error of our single-round algorithm thus matches that of AdaBoost~\eqref{eq:ada}, although making exponentially many queries to the weak learner. As shown by our lower bound, this is inevitable.

Let us also remark that using techniques in the two works~\cite{optimalWeakToStrong,baggingIsOptimal}, we can also remove the two logarithmic factors $\ln(m)\ln(m/d)$ from the upper bound. As the logarithmic factors are not essential to our contribution, we merely comment here that the logarithmic factors can be removed by creating a logarithmic number of bootstrap sub-samples of the training data, running our algorithm in parallel on all sub-samples, and outputting a majority vote among the resulting classifiers. The resulting algorithm is then an optimal weak-to-strong learner by the previously mentioned sample complexity lower bound~\cite{optimalWeakToStrong}.

Finally, let us comment that there are previous boosting algorithms, based on branching programs, that invoke a weak learner in parallel~\cite{DBLP:journals/jcss/MansourM02, DBLP:journals/jcss/KalaiS05,DBLP:conf/colt/LongS05}, however none of these works use $o(\gamma^{-2} \ln m)$ rounds of boosting.

\paragraph{Previous Lower Bounds for Parallelizing Boosting.}
Let us conclude by discussing  related work by Long and Servedio~\cite{long}. In their work, they also study the parallel complexity of boosting, however under somewhat different assumptions. Concretely, they prove a lower bound showing that any weak-to-strong learner must have parallel complexity $(p,t)$ satisfying $p = \Omega(\gamma^{-2} \ln m)$, regardless of the number of calls per round $t$. This strengthens a previous result by Freund~\cite{freund} and is quantitatively a stronger lower bound than ours. However, they also model the problem in a way that makes their result weaker than ours. First, they make no assumption on the complexity/VC-dimension of the hypothesis set used by the weak-learner. On close inspection of their construction, their input domain $\cX$ is the full cube $\{-1,1\}^{k}$ with $k=\Theta(\gamma^{-2} \ln m)$ and they have one hypothesis $h_i$ for each coordinate $i$, making the prediction $h_i(x)=x_i$. The VC-dimension of this hypothesis set is  $\lfloor \lg_2 k \rfloor = \Theta(\ln(1/\gamma) + \ln \ln m)$, i.e. growing with $m$.

Secondly, and just as crucially, their lower bound assumes that the query distributions fed to the weak learner are obtained by "filtering". Concretely, this means that the weight/probability mass put on each sample point $x \in S$ is determined solely from the vector of predictions made by previously obtained hypotheses on $x$ as well as the label $c(x)$ (see their Definition 2). In our lower bound, we make no assumption on how query distributions are chosen other than being computable from the hypotheses seen so far and the training data. This is a crucial difference, and in fact, our 1-round boosting algorithm explicitly queries for distributions that are \emph{not} defined solely from labels and predictions while also using the bounded VC-dimension to define these distributions, thereby circumventing their lower bound.

%% file: lower_arxiv.tex
\section{Impossibility of Parallelization}
In this section, we prove our main result, Theorem~\ref{thm:main}, stating that boosting can not be parallelized significantly without a major reduction in accuracy. We start by presenting the basic setup for our proof, then the high level ideas leading to the lower bound, followed by the formal details.

\paragraph{Basic Setup.}
Let $\cA$ be an arbitrary, possibly randomized, weak-to-strong learning algorithm with parallel complexity $(p,t)$. For any $d,m $ and $\gamma$, we aim to design a $\gamma$-weak learner $\cW$ for a concept $c : \cX \to \{-1,1\}$, using a hypothesis set $\cH$ of VC-dimension $d$, such that $\cA$ run with $\cW$ as its weak learner must have a generalization error at least that stated in Theorem~\ref{thm:main}. We use the simple input domain $\cX = \{x_1,\dots,x_{2m}\}$ and the hard data distribution $\cD$ is the uniform distribution over $\cX$.

Observe that if $p \geq \exp(\Omega(d)), t \geq \exp(\Omega(d/\gamma^2))$ or $t \geq \exp(\exp(\Omega(d)))$, we have nothing left to prove, so assume not. Let $\nu = e^{-O(p \max\{\gamma, \ln(tp)\gamma^2/d\})}$ for short. Our goal is to show that there is a concept $c$ and weak learner $\cW$ such that
\begin{eqnarray}
\label{eq:rand}
    \E_{S,\cA}[\cL_{\cD}(\cA_{S,c(S),\cW})] \geq \nu.
\end{eqnarray}
Here we use $\E_{S,\cA}$ to denote that the expectation is over both the random choice of $S$ and randomness of $\cA$.

\paragraph{Handling Randomness.}
To prove~\eqref{eq:rand}, we consider a \emph{random} choice of concept $c$ and weak learner $\cW$, but a \emph{deterministic} $\cA$. Assume that we can design a distribution over $c$ and $\cW$ (independent of $\cA$), such that any deterministic $\cA$ satisfies
\begin{eqnarray}
\label{eq:det}
    \E_{c,\cW}[\E_{S}[\cL_{\cD}(\cA_{S,c(S),\cW})]] \geq \nu.
\end{eqnarray}
We claim~\eqref{eq:det} implies~\eqref{eq:rand}. To see this, let $\cA$ be an arbitrary randomized weak-to-strong learning algorithm. Then by fixing the random choices of $\cA, c$ and $\cW$ (Yao's principle/linearity of expectation), there is a deterministic $\cA^\star, c^\star$ and $\cW^*$ with
\begin{eqnarray*}
\nu &\leq& \E_{c,\cW}[\E_{S}[\cL_{\cD}(\cA^\star_{S,c(S),\cW})]] \\
&\leq& 
    \E_{c,\cW}[\E_{S,\cA}[\cL_{\cD}(\cA_{S,c(S),\cW})]]\\
    &\leq& 
    \E_{S,\cA}[\cL_{\cD}(\cA_{S,c^\star(S),\cW^\star})].
\end{eqnarray*}
Thus we henceforth focus on proving~\eqref{eq:det}.

\paragraph{Hard Concept Distribution.}
We next describe our hard distribution over a concept $c$ and corresponding $\gamma$-weak learner $\cW$. The distribution over $c$ is simply the uniform distribution, i.e. every $x \in \cX$ is mapped independently to either $-1$ or $1$ with equal probability. This is a natural hard distribution, as under this distribution, the training data $S,c(S)$ only provides information about labels of points in the training data.

Now observe that our data distribution is uniform over $\cX = \{x_1,\dots,x_{2m}\}$ and $S \sim \cD^m$ consists of $m$ samples. Thus there are at least $m$ samples that are not part of the training data. Without querying the weak learner $\cW$, an algorithm $\cA$ has no knowledge of these labels and can at best try to guess them with success probability $1/2$, leading to a constant error probability as a fresh sample $x \sim \cD$ falls outside the training data with probability at least $1/2$.

\paragraph{Hard Weak Learner Distribution.}
In light of the above, we aim to design a $\gamma$-weak learner $\cW$ such that its replies reveal as little as possible information about the labels outside the training data.

For this, we conceptually think of its hypothesis set $\cH$ as consisting of $p$ groups $\cH_1,\dots,\cH_p$. For each of these groups, starting with $\cH_1$, we have a random subset $X_i$ of $\cX$. We let $X_1$ be a uniform random $2m/\beta$ sized subset of $\cX$ and let $X_i$ be a uniform random $2m\beta^{-i}$ sized subset of $X_{i-1}$ for a parameter $\beta>1$ to be determined. We then form $\cH_i$ by drawing $2^{d/2}$ random hypotheses as follows: For each hypothesis $h \in \cH_i$, we let $h(x)=c(x)$ for $x \notin X_i$ and we let $h(x)$ be uniform random and independently chosen for $x \in X_i$. We then add another $2^{d/2}$ random hypotheses to $\cH_i$. These are simply chosen such that they return a uniform random label on every element $x \in \cX$. Finally, we also add $c$ to the hypothesis set $\cH$. Since we assume $p \leq \exp(ad)$ for a sufficiently small constant $a$, we have that the hypothesis set has at most $2^d$ hypotheses and thus has VC-dimension at most $d$.

The weak learner $\cW$ does the following upon being queried for a distribution $\cD'$ over $\cX$: It searches through the hypothesis sets $\cH_i$, starting with $i=1$, and returns the first hypothesis $h$ it sees with $\cL_{\cD'}(h) \leq 1/2-\gamma$. If no such $h$ exists in any $\cH_i$, it simply returns $c$.

\paragraph{Intuition.}
Let us discuss the main ideas in the design of the above (random) weak learner $\cW$. We think of the hypothesis sets $\cH_i$ and corresponding subsets $X_i$ as being responsible for one parallel round each, where $\cH_1,X_1$ is responsible for the first round. 

If we consider the very first round, then observe that $X_1$ is a random $\beta^{-1}$-fraction of $\cX$. Since this subset is unknown to a deterministic $\cA$, if it queries $\cW$ with a distribution $\cD'$ that is somewhat uniform over $\cX$, then about a $1-\beta^{-1}$ fraction of the mass is on points $x \in \cX$ with $x \notin X_1$. Each of the first $2^{d/2}$ hypotheses $h$ in $\cH_1$ has $h(x)=c(x)$ on such points, and thus immediately have an expected advantage of $1-\beta^{-1}$. If $1-\beta^{-1} \approx \gamma$, then there is an overwhelming probability that $\cW$ can answer all queries made in the first round using hypotheses from $\cH_1$.

Now assume that after $i$ rounds, it holds that $\cW$ has only returned hypotheses from $\cH_1,\dots,\cH_i$. Then for round $i+1$, even if $\cA$ learns all hypotheses in $\cH_1,\dots,\cH_i$, the subset $X_{i+1}$ is again an unknown $\beta^{-1}$-fraction of $X_i$. It follows that the weak learner can again answer any queries made by $\cA$ using only $\cH_{i+1}$. Continuing this argument implies that after all $p$ rounds have been completed, the weak learner $\cW$ never returned the hypothesis $c$. Now even if $\cA$ sees all of $\cH_1,\dots,\cH_p$, it still has no knowledge of the labels assigned by $c$ to points in $X_p$. Thus its error probability must be proportional to $|X_p \setminus S|/(2m) = \Omega(\beta^{-p})$. For $1-\beta^{-1} \approx \gamma$, this is $\exp(-O(p \gamma))$ as required in~\eqref{eq:det}.

Finally, let us remark that if $\cA$ queries $\cW$ with a distribution $\cD'$ that is concentrated on a few entries, then the above argument breaks down, i.e. it suddenly becomes rather likely that less than a $1-\beta^{-1}$ fraction of the mass is on points $x \in \cX$ with $x \notin X_i$. In the extreme case, $\cA$ could simply query $\cW$ with every singleton distribution (all mass on one point). This is the reason for including the extra $2^{d/2}$ hypotheses in $\cH_i$ that are simply uniform random. These may be used whenever the support of $\cD'$ is concentrated on $O(d \gamma^{-2})$ entries.

\paragraph{Formal Proof.}
Let $\cA$ be an arbitrary deterministic weak-to-strong learning algorithm. Let 
\[
\eta = \E_{c,\cW}[\E_{S}[\cL_{\cD}(\cA_{S,c(S),\cW})]]. 
\]
We must prove $\eta \geq \nu$ to establish~\eqref{eq:det} and thus Theorem~\ref{thm:main}.

Now let $E$ denote the event that throughout its execution on a random $S \sim \cD^m$,  $\cA$ is never returned the hypothesis $c$ by the weak learner. First we show that the expected error probability $\eta$ is related to $\Pr[E]$
\begin{lemma}
\label{lem:E}
It must be the case that
\begin{eqnarray*}
(\Pr[E]-1/2)2m\beta^{-p} &\leq& 
2 + \lg m + 2m \eta \lg(e \beta^{-p}/\eta).
\end{eqnarray*}
\end{lemma}
Intuitively, if we can show that $\Pr[E]$ is significantly larger than $1/2$, then the term $2m \eta \lg(e \beta^{-p}/\eta)$ must be proportional to $m\beta^{-p}$, thereby providing a lower bound on $\eta$. That $\Pr[E]$ is large is precisely the contents of the next lemma.
\begin{lemma}
\label{lem:largeE}
There are constants $a,a'>0$, such that if $0 < \gamma < a$ and 
\[
\beta = 1 + \max\{32\gamma, a'\ln(tp)\gamma^2/d\},
\]
then $\Pr[E] \geq 49/50$.
\end{lemma}
Let us combine the two lemmas to establish~\eqref{eq:det}. Let $\beta$ be as in Lemma~\ref{lem:largeE}. From Lemma~\ref{lem:E} and Lemma~\ref{lem:largeE}, we see that
\[
(24/50)2m\beta^{-p} \leq 2 + \lg m + 2m \eta \lg(e \beta^{-p}/\eta).
\]
Ignoring the $2 + \lg m$ term, this implies
\[
\eta = \Omega(\beta^{-p}/\lg(e\beta^{-p}/\eta)),
\]
which is equivalent to 
\begin{eqnarray*}
\eta &=& \Omega(\beta^{-p}) \\
&=& \Omega\left(\left(1 + \max\{32\gamma, a' \ln(tp)\gamma^2/d\}\right)^{-p}\right) \\
&\geq& \exp(-O(p\max\{\gamma, \ln(tp)\gamma^2/d\})) \\
&=& \nu.
\end{eqnarray*}
This establishes~\eqref{eq:det} and thus Theorem~\ref{thm:main}.

Let us also remark that for the $2 + \lg m$ term to be relevant, we must have $m\beta^{-p} = O(\lg m)$, which is anyways when the lower bound $\nu = \exp(-O(p \max\{\gamma,\ln(tp) \gamma^2/d\}))$ drops to $1/m^{O(1)}$ and is dominated by the general $\Omega(d/(m \gamma^2))$ lower bound for weak-to-strong learning in previous work~\cite{optimalWeakToStrong}.

What remains is thus to prove Lemma~\ref{lem:E} and Lemma~\ref{lem:largeE}. We start by proving Lemma~\ref{lem:E}.

\paragraph{Relating $\Pr[E]$ and $\eta$.}
The intuition in the proof of Lemma~\ref{lem:E} is that whenever $E$ occurs, the algorithm $\cA$ has no knowledge of $c$ inside $X_p \setminus S$, hence it can only guess these labels, resulting in an error probability of $\Omega(|X_p|/m) = \Omega(\beta^{-p})$.

To formalize this, we use an information theoretic argument. First, we show that even when revealing $S,c(S), \cH_1,\dots,\cH_p,X_1,\dots,X_p$, the random concept $c$ still has a lot of randomness. This randomness is measured in terms of Shannon entropy $H(\cdot)$. Concretely, we show that
\begin{eqnarray}
\label{eq:highentropy}
H(c \mid S, c(S), \cH_1,\dots,\cH_p,X_1,\dots,X_p) \geq m\beta^{-p}.
\end{eqnarray}
Next, we conversely show that if $\eta$ is small and $E$ is likely, then the entropy has to be small.
\begin{lemma}
\label{lem:lowentropy}
The conditional entropy of $c$ is no more than
\begin{eqnarray*}
H(c \mid S, c(S), \cH_1,\dots,\cH_p,X_1,\dots,X_p)  &\leq& 
   2 + \lg m +  (1-\Pr[E])2m\beta^{-p} + 2m\eta \lg(e\beta^{-p}/\eta). 
\end{eqnarray*}
\end{lemma}
Before proving~\eqref{eq:highentropy} and Lemma~\ref{lem:lowentropy}, we show that they together imply Lemma~\ref{lem:E}. Combining the two, we see that
\[
m\beta^{-p} \leq 2 + \lg m +  (1-\Pr[E])2m\beta^{-p} + 2m\eta \lg(e\beta^{-p}/\eta).
\]
Rearranging terms immediately implies Lemma~\ref{lem:E}.

We thus prove~\eqref{eq:highentropy} and Lemma~\ref{lem:lowentropy}. For~\eqref{eq:highentropy}, we have
\begin{eqnarray*}
   H(c \mid S, c(S), \cH_1,\dots,\cH_p,X_1,\dots,X_p) &\geq& 
   \E[|\cH_p \cap (\cX \setminus S)|]. 
\end{eqnarray*}
To see this, note that the conditional entropy of $c$, is the expectation over drawing $S,c(S),\cH_1,\dots,\cH_p,X_1,\dots,X_p$, of the entropy of $c$ conditioned on the outcome. But conditioned on the outcomes of these random variables, $c$ is still uniform random inside $X_p \cap (\cX \setminus S)$ and thus its conditional entropy is at least $|X_p \cap (\cX \setminus S)|$.

We have $\E[|X_p \cap (\cX \setminus S)|] \geq m\beta^{-p}$ from which~\eqref{eq:highentropy} follows.

Next, we prove Lemma~\ref{lem:lowentropy}.
\begin{proof}[Proof of Lemma~\ref{lem:lowentropy}]
We first observe that
\[
\E_{c,\cW,S}[\cL_{\cD}(\cA_{S,c(S),\cW}) \mid E] \leq \eta/\Pr[E].
\]
Now to prove that the conditional entropy of $c$ is small, we give a so-called encoding argument. By Shannon's source coding theorem, the conditional entropy of $c$ is no more than the expected length of a prefix free encoding of $c$. Here the encoding is allowed to make use of the random variables $S, c(S), \cH_1,\dots,\cH_p,X_1,\dots,X_p$. We thus describe such a prefix free encoding and analyse its expected length.

Our encoding procedure is as follows when given $c, S, c(S), \cH_1,\dots,\cH_p,X_1,\dots,X_p$.
\begin{enumerate}
\item First, we check whether the event $E$ occurs on the outcome $c,S,c(S),\cH_1,\dots,\cH_p,X_1,\dots,X_p$. This can be checked since $\cW$ can be simulated from $\cH_1,\dots,\cH_p$ as it always returns the first hypothesis that has $\gamma$ advantage on a query distribution $\cD$. Furthermore, $\cA$ is deterministic, hence we can also simulate $\cA$ from $S,c(S)$ and the replies by $\cW$. Therefore, we can let the first bit of the encoding be $1$ if $E$ occurs and $0$ otherwise.
\item If $E$ did not occur, we finish the encoding by appending $|X_p|$ bits specifying $c(X_p)$. This takes $2m\beta^{-p}$ bits. 
\item If $E$ occurs, we simulate $\cA$ with $\cW, c(S),S$ to produce the hypothesis $h = \cA_{S,c(S),\cW}$. We then evaluate $h$ on $X_p$ and let $Z$ denote the set of $x \in X_p$ for which $h(x) \neq c(x)$. Let $\Delta = |Z|$. Our encoding now specifies $\Delta$ by appending $\lg(2m) = 1 + \lg m$ bits. Next, we specify $Z$ as a $\Delta$-sized subset of $X_p$, costing $\lg \binom{|X_p|}{\Delta} = \lg \binom{2m \beta^{-p}}{\Delta} \leq \Delta \lg(2e m \beta^{-p}/\Delta)$ bits.
\end{enumerate}

We next describe how to decode $c$ from the above encoding and $S,c(S),\cH_1,\dots,\cH_p,X_1,\dots,X_p$.
\begin{enumerate}
\item We start by checking the first bit of the encoding. If this is a $0$-bit, we read the remaining $2m\beta^{-p}$ bits to recover $c(X_p)$. Finally, all remaining $c(x)$ can be reconstructed from the first hypothesis in $\cH_p$.
\item If the first bit is $1$, we read the second part of the encoding to reconstruct the set $Z$. We then simulate $\cA$ and $\cW$ to obtain $h = \cA_{S,c(S),\cW}$. Next, we evaluate $h$ on every $x \in X_p$ and correct the mistakes using $Z$. Again, we can also recover $c(x)$ for every $x \notin X_p$ from $\cH_p$.
\end{enumerate}

The above procedures thus give a prefix free encoding of $c$ conditioned on $S, c(S), \cH_1,\dots,\cH_p,X_1,\dots,X_p$. We now analyse its expected length to derive an upper bound on the conditional entropy of $c$. We have that the expected length of the encoding is
\begin{eqnarray*}
1 + (1-\Pr[E])2m\beta^{-p} +
\Pr[E](1 + \lg m + \E[\Delta \lg(2e m \beta^{-p}/\Delta) \mid E])
\end{eqnarray*}
Note that $x \lg(a/x) = \lg(e) (x \ln(a/x)) = \lg(e)(x \ln(a) - x\ln(x))$ has derivative $\lg(e)(\ln(a) -(\ln(x)+1))$ and second derivative $-\lg(e)/x$. Thus it is a concave function and $\E[ \Delta \lg(2e m \beta^{-p}/\Delta) \mid E] \leq \E[\Delta \mid E]\lg(2e m \beta^{-p}/\E[\Delta \mid E])$. This is an increasing function of $\E[\Delta \mid E]$ (note that $\Delta \leq 2m \beta^{-p}$) and thus we conclude
\begin{eqnarray*}
H(c \mid S, c(S), \cH_1,\dots,\cH_p,X_1,\dots,X_p) &\leq& \\
2 + \lg m +  (1-\Pr[E])2m\beta^{-p} +
\Pr[E](2m\eta/\Pr[E]) \lg(2e m \Pr[E]\beta^{-p}/(2m\eta)) &\leq& \\
2 + \lg m +  (1-\Pr[E])2m\beta^{-p} + 2m\eta \lg(e\beta^{-p}/\eta).
\end{eqnarray*}
\end{proof}

\paragraph{The Event $E$ is Likely.}
What remains is to prove Lemma~\ref{lem:largeE}, i.e. show that $\Pr[E] \geq 49/50$ for the choice of $\beta$ in Lemma~\ref{lem:largeE}. For this, let $E_i$ denote the event that in the first $i$ rounds of querying the weak-learner $\cW$, it holds that $\cW$ only returns hypotheses from $\cH_1,\dots,\cH_{i}$. 

We aim to show that $\Pr[E_i \mid \cap_{j=1}^{i-1} E_j]$ is large. For this, notice that conditioned on $\cap_{j=1}^{i-1} E_j$ and any outcome of $S,c(S),X_1,\dots,X_{i-1}, \cH_1,\dots,\cH_{i-1}$, all the $t$ queries made by $\cA$ in round $i$ are fixed. This is because we can simulate the replies of the weak learner $\cW$ during the first $i-1$ rounds from this information alone, as it always manages to return a hypothesis from $\cH_1,\dots,\cH_{i-1}$ when we condition on $\cap_{j=1}^{i-1} E_j$. Let $\cQ_i(S,c(S),X_1,\dots,X_{i-1}, \cH_1,\dots,\cH_{i-1})$ denote the set of queries made by $\cA$ in round $i$ on $S,c(S),X_1,\dots,X_{i-1}, \cH_1,\dots,\cH_{i-1}$. Thus $\cQ_i$ is a set of distributions.

Now observe that conditioned on any outcome of $S,c(S),X_1,\dots,X_{i-1}, \cH_1,\dots,\cH_{i-1}$, the set $X_i$ is a uniform random $\beta^{-1}|X_{i-1}|$-sized subset of $X_{i-1}$. Thus for a query distribution $\cD' \in \cQ_i$, we can bound the probability that there is no $h \in \cH_1,\dots,\cH_i$ with advantage $\gamma$ for $\cD'$. 

We split the proof in two parts. 

\paragraph{Well-Spread Queries.}
In the first case, there is no set of $\rho = a d\gamma^{-2}$ elements receiving at least $1/4$ of the probability mass under $\cD'$, where $a>0$ is a sufficiently small constant.

We say that $X_i$ \emph{avoids} $\cD'$ if $\Pr_{x \sim \cD'}[x \notin X_i] \geq 2\gamma$. We claim that if $X_i$ avoids $\cD'$, then there is an $h \in \cH_i$ with $\cL_{\cD'}(h) \leq 1/2-\gamma$ with probability $1-\exp(-\exp(\Omega(d)))$. We start by proving this.

Assume $X_i$ avoids $\cD'$ and let $h$ be one of the $2^{d/2}$ random hypotheses from $\cH_i$ that equal $c$ for all $x \notin X_i$. Since $h(x)$ is uniform in $\{-1,1\}$ for all $x \in X_i$, it holds with probability $1/2$ that $\sum_{x \in X_i}c(x)h(x) \geq 0$. When this happens, we have $\E_{x \sim \cD'}[h(x)c(x)] \geq \sum_{x \notin X_i} \cD'(x)c(x)h(x) = \sum_{x \notin X_i} \cD'(x) = \Pr_{x \sim \cD'}[x \notin X_i] \geq 2 \gamma$. This implies that $\cL_{\cD'}(h) \leq 1/2 - \gamma$ and thus $h$ is a valid response to the query $\cD'$. Since this holds with probability $1/2$ for each of the $2^{d/2}$ random hypotheses, it holds with probability $1-2^{-2^{d/2}} = 1-\exp(-\exp(\Omega(d)))$ that there is a hypothesis $h \in \cH_i$ that is a valid response to $\cD'$ for $\cW$. We thus need to bound the probability that $X_i$ avoids $\cD'$. 

Consider first the largest $\rho$ elements under $\cD'$. Denote this set $Y$. We know that $\cD'$ puts at least $3/4$ probability mass outside $Y$ and that no element outside has mass more than $1/\rho$.

Let $Z=\cX \setminus (X_i \cup Y)$. Our goal is to show that with high probability (over $X_i$), we have $\Pr_{x \sim \cD'}[x \in Z] \geq 2 \gamma$, which implies that $X_i$ avoids $\cD'$.

Define an indicator $I_x$ for every $x \in \cX \setminus Y$, taking the value $1$ if $x \notin X_{i}$ and $0$ otherwise. We want to show that $\Pr_{X_i}[\Pr_{x \sim \cD'}[x \in Z] < 2\gamma] = \Pr_{X_i}[\sum_{x \in \cX \setminus Y} I_x \cD'(x) < 2\gamma]$ is small. We start by showing that $\mu = \E_{X_i}[\sum_{x \in \cX \setminus Y} I_x \cD'(x)]$ is large. Here we notice that if $x \notin X_{i-1}$, then $I_x = 1$ with probability $1$. For $x \in X_i$, we have $\Pr[I_x = 1] = 1-|X_i|/|X_{i-1}| = 1-\beta^{-1}$. We thus have $\mu \geq (1-\beta^{-1})\sum_{x \in \cX \setminus Y}\cD'(x) \geq (3/4)(1-\beta^{-1})$. Since Lemma~\ref{lem:largeE} constrains $\beta \geq 1 + 32\gamma$, we have $2\gamma \leq (\beta - 1)/16 = (\beta^{-1}\beta-\beta^{-1})/(16 \beta^{-1}) = (1-\beta^{-1})/(16 \beta^{-1}) \leq (1-\beta^{-1})/16$. We thus need to bound the probability that the sum drops to less than half of its expectation.

Now consider the population $\{\cD'(x) : x \in X_{i-1}\setminus Y\}$ and add it to it $|Y \cap X_{i-1}|$ $0$'s. Then $\sum_{x \in X_{i-1} \setminus Y} I_x \cD'(x)$ is distributed as the sum of $\Gamma=|X_{i-1}| - |X_i|$ samples $R_1,\dots,R_\Gamma$ without replacement from this population. Furthermore, every element in the population is bounded by $1/\rho$ in value. In the appendix, Section~\ref{sec:withoutRep}, we show how to use a version of the Chernoff bound for sampling without replacement to bound $\Pr[\sum_i R_i < \mu/2]$ by $\exp(-\rho \mu/8)$.

Since we defined $\rho = a d\gamma^{-2}$, this probability is at most 
$
\exp\left(-\Omega(d(1-\beta^{-1})\gamma^{-2})\right) 
$.

A union bound over all $t$ queries $\cD'$ in $\cQ_i$ implies that $X_i$ avoids them all with probability at least $1-t\exp\left(-\Omega(d(1-\beta^{-1})\gamma^{-2})\right)$. When this happens, another union bound implies that there is a valid response to all such $\cD'$ with probability $1-t\exp(-\exp(\Omega(d)))$.

\paragraph{Concentrated Queries.}
Next, consider a $\cD'$ with some set $Y$ of $\rho = ad\gamma^{-2}$ elements receiving at least $1/4$ probability mass. For such $\cD'$, we consider the $2^{d/2}$ random $h \in \cH_i$ that return a uniform random value for every $x \in \cX$. We want to show that each of these is reasonably likely to have the desired advantage of $\gamma$ on $\cD'$. For this, first define the following function of a vector $w \in \R^n$ and value $t > 0$:
\[
F(w,t) = \sum_{i=1}^{\lfloor t^2 \rfloor} |w_{(i)}| + t \left(\sum_{j=\lfloor t^2 \rfloor+1}^n w_{(j)}^2 \right)^{1/2}.
\]
Here $w_{(i)}$ denotes the $i$'th largest entry of $w$ in absolute value. 

With this definition, The following theorem by Montgomery-Smith shows that it is rather likely that any such hypothesis $h$ provides the desired advantage
\begin{theorem}[\cite{montgomery}]
\label{thm:mont}
There exists universal constants $a_1,a_2 > 0$, such that the following holds: For any vector $w \in \R^n$, if $x$ has uniform random and independent entries in $\{-1,1\}$, then for all $t > 0$:
$
    \Pr\left[\langle w, x \rangle > a_1 F(w,t)\right] \geq a_2^{-1}\exp(-a_2 t^2)$.
\end{theorem}
Now let $w$ be the vector in $\R^\rho$ with entries containing the values $\cD'(x)c(x)$ for $x \in Y$. For the random $h$, we invoke Theorem~\ref{thm:mont} with $t = a_3 \sqrt{d}$ for a small constant $a_3 > 0$. Assume first that the largest $\lfloor t^2 \rfloor$ entries of $w$ sum to more than $1/8$. Then for $\gamma < 1/16$, we immediately get from Theorem~\ref{thm:mont} that $\sum_{x \in Y} \cD'(x)c(x)h(x) \geq 2\gamma$ with probability at least $\exp(-O(t)) = \exp(-O(a_3^2 d))$. If on the other hand the largest entries sum to less than $1/8$, then the remaining sum to at least $1/8$. Since $\R^\rho$ has $\rho$ coordinates, it follows by Cauchy-Schwartz that $\sqrt{\sum_{j=\lfloor t^2 \rfloor+1}^\rho w_{(j)}^2} \geq 1/(8\sqrt{\rho})$ and thus from Theorem~\ref{thm:mont}, we get $\sum_{x \in Y} \cD'(x)c(x)h(x) = \Omega(t/\sqrt{\rho}) = \Omega(a_3\sqrt{d}/\sqrt{a d \gamma^{-2}}) \geq 2\gamma$ with probability $\exp(-O(t)) = \exp(-O(a_3^2 d))$. Here we use that $a>0$ is sufficiently small compared to $a_3$. Finally note that $\sum_{x \notin Y} \cD'(x)h(x)c(x)$ is at least $0$ with probability $1/2$ and is independent of $\sum_{x \in Y}\cD'(x)h(x)c(x)$.

The probability that there is no valid response to $\cD'$ in $\cH_i$ is thus at most $(1-\exp(-O(a_3^2 d)))^{2^{d/2}} \leq \exp(-\exp(\Omega(d)))$ for $a_3$ sufficiently small. A union bound over all such $\cD'$ in $\cH_i$ finally shows that $\cH_i$ has a valid response $h$ to each of them with probability at least $1-t\exp(-\exp(\Omega(d)))$.

\paragraph{Summary.}
We finally conclude from the above that
\begin{eqnarray*}
 \Pr[E_i \mid \cap_{j=1}^{i-1} E_j] &\geq&
 1-t(\exp(-\exp(\Omega(d))) +
 \exp(-\Omega(d(1-\beta^{-1})\gamma^{-2}))).   
\end{eqnarray*}
which gives us
\begin{eqnarray*}
\Pr[E] \geq
\left(1-t(e^{-\exp(\Omega(d))} + e^{-\Omega(d(1-\beta^{-1})\gamma^{-2})})\right)^p 
\geq
1-tp(e^{-\exp(\Omega(d))} + e^{-\Omega(d(1-\beta^{-1})\gamma^{-2})}).
\end{eqnarray*}
Lemma~\ref{lem:largeE} sets $\beta \geq 1 + a' 
\ln(tp)\gamma^2/d$ for a large constant $a'$. We have $1-\beta^{-1} = (\beta - 1)/\beta$. This is increasing in $\beta$ and thus is at least $(a' \ln(tp)\gamma^2/d)/(1 + a'
\ln(tp)\gamma^2/d)$. Since we assume $tp = \exp(O(d\gamma^{-2}))$, we have $ a'
\ln(tp)\gamma^2/d \leq 1$ for sufficiently small constant in the $O$-notation (depending on $a'$). Thus $1-\beta^{-1} \geq (a'/2) \ln(tp)\gamma^2/d$. This finally implies $e^{-\Omega(d(1-\gamma^{-1}) \gamma^{-2})} \leq 1/(100tp)$ for big enough $a'$. By assumption, we also have $tp \leq \exp(\exp(O(d)))$, implying $tp e^{-\exp(\Omega(d))} \leq 1/100$ and thus $\Pr[E] \geq 49/50$ as desired. This concludes the proof of Lemma~\ref{lem:largeE} and thereby also the proof of Theorem~\ref{thm:main}.

%% file: upper_arxiv.tex
\section{Single-Round Boosting}
In this section, we demonstrate a single-round parallel boosting algorithm with good generalization performance. Our algorithm initially asks a number of queries in parallel to a $\gamma$-weak learner. Following that, it has several sequential rounds in which it decides how to combine the obtained hypotheses. We comment that this is a single-round algorithm according to our definition of parallel complexity. Moreover, invoking the weak learner effectively corresponds to training a model, whereas the later sequential steps of our algorithm uses only inference. In practical setups where training is significantly more expensive than inference, this may still be a major speed up. Furthermore, our algorithm demonstrates the near-tightness of our lower bound.

The goal of our algorithm is to produce a voting classifier with large margins on all training samples. Concretely, given a $\gamma$-weak learner $\cW$ and $m$ samples $S \sim \cD^m$, our goal is to produce a voting classifier $f(x) = \sign(g(x))$ with $g(x) = (1/k)\sum_{i=1}^k h_i(x)$ such that for every training sample $(x,c(x)) \in S$, it holds that $c(x)g(x) \geq \gamma/16$. Generalization of this voting classifier then follows from generalization bounds for voting classifiers with large margins, see e.g.~\cite{breiman1999prediction,gao,bartlett}. We note that we assume that the weak learner $\cW$ \emph{always} returns a hypothesis with error at most $1/2-\gamma$ under the distribution $\cD$ it is queried with.

\paragraph{Algorithm.}
For a large enough constant $a > 0$, our algorithm simply queries the weak learner $\cW$ for every distribution $\cD_T$ with $T \subseteq S$, $|T|=a d \gamma^{-2}$, where $T$ is a multiset. Here $\cD_T$ is the uniform distribution over $T$, putting a mass of $t/|T|$ on an element that occurs $t$ times in the multiset. These queries can clearly be performed in parallel and there are no more than $m^{ad\gamma^{-2}} = \exp(O(d \gamma^{-2} \ln m))$ such queries.

Let $h_T \in \cH$ denote the hypothesis returned by $\cW$ on the query $\cD_T$. Having obtained the set of all $h_T$, we form $g$ by running a version of AdaBoost shown in Algorithm~\ref{alg:adaboost}.

\input{alg_adaboost_arxiv.tex}

We invoke Algorithm~\ref{alg:adaboost} with the training data set $S$ as well as a number of rounds $K = 16 \gamma^{-2} \ln m$.

We make a few remarks regarding the algorithm. First note that it is possible to check whether $\cL_{\cD_k}(h_k) > 1/2-\gamma/4$ because $\cD_k$ is supported only on the training data, and thus the error probability of $h_k$ under $\cD_k$ can be computed exactly.

The reader familiar with AdaBoost will also notice that the weight changes we make, i.e. the chosen $w$, is consistent with AdaBoost's weight updates if the hypothesis $h_k$ was correct on precisely a $1/2+\gamma/4$ fraction of the samples when weighted according to $\cD_k$. Our algorithm is thus a variant of AdaBoost with uniform weight updates and weighing of hypotheses in the output $g$.

\paragraph{Analysis.}
We analyse our algorithm in two steps. First, we show that if the algorithm terminates, then the resulting voting classifier $g = (1/K)\sum_{i=1}^K h_i(x)$ has all margins at least $\gamma/16$. Next, we show that the algorithm always terminates. These two properties are captured in the following lemmas
\begin{lemma}
\label{lem:goodout}
If Algorithm~\ref{alg:adaboost} terminates, then the resulting voting classifier $g$ satisfies $c(x)g(x) \geq \gamma/16$ for all $(x,c(x)) \in S$.
\end{lemma}

\begin{lemma}
\label{lem:terminate}
Algorithm~\ref{alg:adaboost} always terminates.
\end{lemma}

The main observation for the proof of Lemma~\ref{lem:goodout} is that whenever our algorithm terminates, every hypothesis $h_k$ that it uses, has $\cL_{\cD_k}(h_k) \leq 1/2 - \gamma/4$. Furthermore, we use the same weight adjustment $w$ in all iterations and weigh the $h_k$ equally in the output $g$. A previous analysis~\cite{boostingLowerBound} of AdaBoost with identical weights has also shown that this results in margins of $\Omega(\gamma)$. 

\begin{proof}[Proof of Lemma~\ref{lem:goodout}]
Consider the exponential loss 
\[
\sum_{i=1}^m \exp\left(-w c(x_i)\sum_{k=1}^K h_k(x_i)\right).
\]
We compare this to the final weights $\cD_{K+1}$. Here we note that
\begin{eqnarray*}
1 &=& \sum_{i=1}^m \cD_{K+1}(i) \\
&=& \sum_{i=1}^m \frac{\cD_{K}(i) \exp(-wc(x_i)h_K(x_i))}{Z_k} \\
&=& \frac{1}{m} \sum_{i=1}^m \frac{\exp(-w c(x_i) \sum_{k=1}^K h_k(x_i)}{\prod_{k=1}^K Z_k}.
\end{eqnarray*} 
From this, we observe that 
\[
\sum_{i=1}^m \exp\left(-w c(x_i)\sum_{k=1}^K h_k(x_i)\right) = m \prod_{k=1}^K Z_k.
\]
Next, let $er_k = \Pr_{x \sim \cD_k}[h_k(x) \neq c(x)]$ and recall $er_k \leq 1/2-\gamma/4$. Observe that
\begin{eqnarray*}
Z_k &=& \sum_{i=1}^m \cD_k(i) \exp(-w c(x_i)h_k(x_i)) \\
&=& \sum_{i : h_k(x_i)\neq c(x_i)} \cD_k(i)e^{w} + \sum_{i : h_k(x_i) = c(x_i)} \cD_k(i) e^{-w} \\
&=& er_k e^w + (1-er_k)e^{-w} \\
&\leq& (1/2-\gamma/4) e^w + (1/2+\gamma/4)e^{-w} \\
&=& 2 \sqrt{(1/2-\gamma/4)(1/2+\gamma/4)} \\
&=& \sqrt{1 - \gamma^2/4}.
\end{eqnarray*}
Since $K = 16 \gamma^{-2} \ln m$, we have $\prod_{k=1}^K Z_k \leq (1-\gamma^2/4)^{K/2} \leq \exp(-\gamma^2 K /8) \leq 1/m^2$. We therefore have 
\[
\sum_{i=1}^m \exp\left(-w c(x_i)\sum_{k=1}^K h_k(x_i)\right) \leq 1/m.
\]
By non-negativity of the exponential function, this in particular implies that $\exp(-wc(x_i) \sum_{k=1}^K h_k(x_i)) \leq 1/m$ for all $i$. Raising both sides of the inequality to the power $1/(Kw)$ gives $\exp(-c(x_i)g(x_i)) \leq 1/m^{1/Kw}$. Taking inverses and log gives $c(x_i)g(x_i) \geq \ln(m)/(Kw)$. Note that 
\begin{eqnarray*}
    w &=& \frac{1}{2}\ln((1/2 + \gamma/4)/(1/2-\gamma/4)) \\
    &\leq& \frac{1}{2}\ln(1 + 2 \gamma) \\
    &\leq& \gamma.
\end{eqnarray*}
Hence we conclude $c(x_i)g(x_i) \geq \ln(m)/(Kw) \geq \gamma/16$.
\end{proof}

To prove termination (Lemma~\ref{lem:terminate}), we first need to recall the notion of an $\eps$-approximation. For a concept $c : \cX \to \{-1,1\}$, a hypothesis set $\cH$ and a distribution $\cD$ over $\cX$, a set of samples $T$ is an $\eps$-approximation for $(c,\cD,\cH)$ if for all $h \in \cH$, it holds that 
\[
|\Pr_{x \sim \cD}[h(x) \neq c(x)] - |\{ x \in T : h(x)\neq c(x)\}|/|T|| \leq \eps.
\]
The following classic result shows that a small random sample $T \sim \cD^n$ is an $\eps$-approximation with good probability.
\begin{theorem}[\cite{lls,talagrand,vapnik71uniform}]
\label{thm:sample}
There is a universal constant $b > 0$, such that for any $0 < \eps,\delta < 1$, $\cH \subseteq \cX \to \{-1,1\}$ of VC-dimension $d$ and distribution $\cD$ over $\cX$, it holds with probability at least $1-\delta$ over a set $T \sim \cD^n$ that $T$ is an $\eps$-approximation for $(c,\cD,\cH)$ provided that $n \geq b((d+\ln(1/\delta))\eps^{-2})$.
\end{theorem}

With this, we are ready to prove that Algorithm~\ref{alg:adaboost} terminates.

\begin{proof}[Proof of Lemma~\ref{lem:terminate}]
We need to show that the while-loop always terminates. So consider some distribution $\cD_k$ in round $k$ of Algorithm~\ref{alg:adaboost}. We claim the while-loop terminates if the sample $T_k$ is a $\gamma/2$-approximation for $(c,\cD_k,\cH)$. To see this, note that if it is a $\gamma/2$-approximation, then
\[
\left|\Pr_{x \sim \cD_k}[h(x) \neq c(x)] - \frac{|\{ x \in T_k : h(x)\neq c(x)\}|}{|T_k|}\right| \leq \gamma/2.
\]
Furthermore, we have that $h_k = h_{T_k}$ was obtained from the weak learner $\cW$ using the distribution $\cD_{T_k}$. Thus 
\begin{eqnarray*}
    1/2 - \gamma \leq
    \Pr_{x \sim \cD_{T_k}}[h_k(x) \neq c(x)] =
    |\{ x \in T_k : h_k(x)\neq c(x)\}|/|T_k|.
\end{eqnarray*}
It follows that $\Pr_{x \sim \cD_k}[h_k(x) \neq c(x)] \leq (1/2-\gamma) + \gamma/2 = 1/2 - \gamma/2$. The conclusion now follows from Theorem~\ref{thm:sample} which says that we have a $\gamma/2$-approximation with constant probability in every iteration of the while-loop, provided that $n \geq b(d + 2)\gamma^{-2}/4$, which is indeed satisfied for our choice of $n = ad\gamma^{-2}$ for $a$ large enough.
\end{proof}

\paragraph{Generalization Performance.}
We finally conclude by observing that our algorithm always produces a voting classifier $f = \sign(g(x))$ where $c(x)g(x) \geq \gamma/16$ for all $(x,c(x)) \in S$ (Lemma~\ref{lem:goodout} and Lemma~\ref{lem:terminate}). We now invoke Breiman's min-margin bound
\begin{theorem}[\cite{breiman1999prediction}]
\label{thm:margin}
Let $c : \cX \to \{-1,1\}$ be an unknown concept, $\cH \subseteq \cX \to \{-1,1\}$ a hypothesis set of VC-dimension $d$ and $\cD$ an arbitrary distribution over $\cX$. There is a universal constant $a > 0$ such that with probability at least $1-\delta$ over a set of $m$ samples $S \sim \cD^m$, it holds for every voting classifier $f(x) = \sign(g(x))$ with $c(x)g(x) \geq \gamma$ for all $(x,c(x)) \in S$ that
\[
\cL_\cD(f) \leq a \cdot \frac{d \ln(m) \ln(m/d) + \ln(1/\delta)}{\gamma^2 m}.
\]
\end{theorem}
The generalization of our single-round boosting algorithm follows immediately from Theorem~\ref{thm:margin} and the observations above.

%% file: alg_adaboost_arxiv.tex
\begin{algorithm2e}[t]
  \DontPrintSemicolon
  \KwIn{Training set $S = \{(x_1,c(x_1)),\dots,(x_m,c(x_m))\}$,\\
    \quad number of rounds $K$\\
    % \quad desired accuracy $\nu$ \\
    }
  \KwResult{A voting  classifier $g$}
  \SetKw{Break}{break}
  \SetKw{KwReturn}{return}

  $ \cD_1 \gets \left( \frac{1}{m},\dots \frac{1}{m} \right)$ 
  
    $w \gets \frac{1}{2} \ln ({\frac{1/2+\gamma/4}{1/2-\gamma/4}})$
    
    $n \gets a d \gamma^{-2}$
    
  \For(){$k= 1,\dots,K$}{
  
    Draw $n$ samples $T_k \sim \cD_k^{n}$ 
    
    Let $h_k \gets h_{T_k}$
    
    \While(){$\cL_{\cD_k}(h_k)> 1/2-\gamma/4$}{
    
        Re-draw $n$ samples $T_k \sim \cD_k^n$
        
        Let $h_k \gets h_{T_k}$
    }
    
    \For(){$i \in \{1,\dots,m\}$}{
        $ \cD_{k+1}(i) \gets \cD_k(i) \exp(-w c(x_i) h_k(x_i))$
    }

    $Z_k \gets \sum_{i=1}^m \cD_k(i)\exp(-w c(x_i) h_k(x_i))$

    $\cD_{k+1} \gets \cD_{k+1}/Z_k$
  }
  \KwReturn{$g(x) = \frac{1}{K}\sum_{i=1}^K h_i(x)$}
  % \tcp*{weighted majority vote}
        
  \caption{Sampled Boosting}\label{alg:adaboost}
\end{algorithm2e}

%% file: concl_arxiv.tex
\section{Conclusion}
In this work, we established strong barriers for the possibilities of parallelizing boosting. Concretely, to parallelize below $O(\gamma^{-1} \ln m)$ rounds incurs an exponential blow-up in the number of invocations of a weak learner. 

The classic algorithms, such as AdaBoost, use $O(\gamma^{-2} \ln m)$ rounds. Thus it is conceivable that boosting can be somewhat parallelized. We leave this as an exciting direction for future research.

We also complemented our lower bound by a near-optimal single-round boosting algorithm. A slightly insatisfactory aspect of our single-round algorithm is that performs sequential work in $O(\gamma^{-2} \ln m)$ rounds after having invoked the weak learner/trained a number of hypotheses. It would be interesting to obtain a truly parallel algorithm also when accounting for post-processing of trained models.

%% file: ack_arxiv.tex
\section*{Acknowledgment}
The authors would like to thank Reza Shokri, Diptarka Chakraborty and the National University of Singapore for inviting both authors as speakers at the NUS CS Research Week 2023. We thank Reza, Diptarka and NUS for introducing us during this visit, thereby initiating the collaboration leading to the results in this paper.

%% file: appendix_arxiv.tex
\section{Sampling Without Replacement}
\label{sec:withoutRep}
We prove the following lemma for sampling without replacement from a finite population
\begin{lemma}
Let $Y_1,\dots,Y_n$ be samples without replacement from a finite population of real values between $0$ and $1/\rho$ for $\rho > 0$. Then for any $0< \delta < 1$ and any $\mu \leq \E[\sum_i Y_i]$, it holds that
\[
\Pr[\sum_i Y_i \leq (1-\delta)\mu] \leq \exp(-\rho \delta^2 \mu/2)
\]
and for any $\gamma \geq \E[\sum_i Y_i]$ it holds that
\[
\Pr[\sum_i Y_i \geq (1+\delta)\mu] \leq \exp(-\rho \delta^2 \mu/3).
\]
\end{lemma}
\begin{proof}
We start by rescaling all elements in the population by a factor $\rho$. This changes $\E[\sum_i Y_i]$ by a factor $\rho$ and guarantees that all values are between $0$ and $1$. We thus need to show that the two probabilities are bounded by $\exp(- \delta^2 \mu/2)$ and $\exp(-\delta^2 \mu/3)$ for these rescaled variables. Notice that these two are precisely the standard Chernoff bounds for sums of independent 0/1 random variables. The proof of the Chernoff bounds follows from proving that 
\begin{eqnarray}
\label{eq:chern}
\E[\exp(t \sum_i Y_i)] \leq \prod_i \left(1 + \E[Y_i](e^t-1)\right).
\end{eqnarray}
for any $t \in \R$. We thus prove~\eqref{eq:chern} for our rescaled variables and refer the reader to standard proofs of the Chernoff bound from thereon.

Now let $X_1,\dots,X_n$ be independent and uniform random samples with replacement from the population. Then $\E[\sum_i Y_i] = \E[\sum_i X_i]$.
Since $x \to \exp(t x)$ is a convex and continous function for any $t \in \R$, it follows from Hoeffding~\cite{Hoeffding} that $\E[\exp(t \sum_i Y_i)] \leq \E[\exp(t \sum_i X_i)] = \prod_i \E[\exp(t X_i)]$. Again using convexity of $\exp(t x)$, we get that for any $x \in [0,1]$, it holds that $\exp(t x) \leq (1-x)\exp(t 0) + x\exp(t 1)$. Since we rescaled the population, every $X_i$ takes values between $0$ and $1$. Hence $\E[\exp(t X_i)] \leq \E[(1-X_i) + X_i \exp(t)] = 1 + \E[X_i](e^t-1) = 1 + \E[Y_i](e^t-1) $.
\end{proof}